\documentclass[runningheads,a4paper]{llncs}

\usepackage{amssymb}
\usepackage{graphicx}
\usepackage{algorithm}
\usepackage{algpseudocode}

\usepackage{url}
\urldef{\mailsa}\path|{shriprakash.sinha@gmail.com, g.j.ter.horst@med.umcg.nl}|
\newcommand{\keywords}[1]{\par\addvspace\baselineskip
\noindent\keywordname\enspace\ignorespaces#1}

\begin{document}

\mainmatter  % start of an individual contribution

\title{Bounded Multivariate Surfaces On Monovariate Internal Functions}

\titlerunning{Bounded Multivariate Surfaces On Monovariate Internal Functions}

\author{Shriprakash Sinha \and Gert J. ter Horst%
\thanks{This work was conducted as a part of ongoing research in segmenting and labeling brain images at the Neuroimaging Center, in University Medical Center Groningen, The Netherlands. Shiprakash Sinha is with the Neuroimaging Center at University Medical Center Groningen, Antonius Deusinglaan 2, 9713 AW, Groningen, The Netherlands. Dr. Gert J. ter Horst is the General Director of the Neuroimaging Center, UMCG.}%
}
\authorrunning{Sinha and Horst}

\institute{Neuroimaging Center,  Antonius Deusinglaan 2 \\University
  Medical Center Groningen, 9713 AW, Groningen, The Netherlands\\
\mailsa\\
%\url{http://www.bcn-nic.nl/}
}


%\toctitle{Lecture Notes in Computer Science}
%\tocauthor{Authors' Instructions}
\maketitle

\begin{abstract}
Combining the properties of monovariate internal functions as proposed in Kolmogorov superimposition theorem, in tandem with the bounds wielded by the multivariate formulation of Chebyshev inequality, a hybrid model is presented, that decomposes images into homogeneous probabilistically bounded multivariate surfaces. Given an image, the model shows a novel way of working on reduced image representation while processing and capturing the interaction among the multidimensional information that describes the content of the same. Further, it tackles the practical issues of preventing leakage by bounding the growth of surface and reducing the problem sample size. The model if used, also sheds light on how the Chebyshev parameter relates to the number of pixels and the dimensionality of the feature space that associates with a pixel. Initial segmentation results on the Berkeley image segmentation benchmark indicate the effectiveness of the proposed decomposition algorithm.\keywords{Multivariate Inequality, Monovariate functions, Bounded Surfaces} 
\end{abstract}

\section{Introduction}
\label{sec:intro}

In order for an image to be decomposed, a proper representation of the image must first be done. One set of solutions for image representation is the decomposition of multivariate functions into monovariate functions as proposed by Kolmogorov superimposition theorem (KST) \cite{Sprecher:1972}. Sprecher et.al \cite{Sprecher:2002} have also proved that the monovariate internal functions obtained via the KST can be used to build space filling curves that sweep a multidimensional space. These 1D representations of the image can then be exploited for further processing using simple univariate or bivariate signal processing methods, as has been shown in \cite{Leni:2008} and \cite{Leni:2009}. It has further been proposed in \cite{Leni:2009} that either the space filling curves can be fixed and then construct external function whoes sum and compositions correspond to a multivariate function \cite{Sprecher:2002} or produce an algorithm that generates the internal function that adapts to the multivariate function and gives different space filling curve for different multivariate functions \cite{Igelnik:2003}. The work in this manuscript finds its motivation in presenting an intial hybrid model that uses a fixed space filling curve for an image and differs in employing mutivariate Chebyshev inequality on the $\mathcal{N}$ dimensional points lying on the curve, to yield homogeneous probabilistically bounded multivariate surfaces. \par
From the theory of space filling curves, it is known that the Hilbert Space Filling Curve (HSFC) \cite{Hilbert:1891} is the best in preserving the clustering properties while taking into account the locality of objects in a multidimensional space (\cite{Jagadish:1990}, \cite{Moon:2001}). Even though it can be applied to transform a multidimensional image representation into a linear format (\cite{Leni:2008}, \cite{Leni:2009}), the manuscript applies the HSFC to transform a $2$D matrix into $1$D space filling curve. The reason being that it saves the time and avoids the complexity in processing a $2$D matrix in comparison to $\mathcal{N}$D matrix. \par
Next, a multivariate formulation of the generalized Chebyshev inequality \cite{Chen:2007} is applied to decompose the image into surfaces bounded probabilistically via a single Chebyshev parameter. Since the bounded surfaces are constructed based on the interaction of a set of points in $\mathcal{R}^{\mathcal{N}}$ lying on the curve, it can be safely assumed that information about the nature or density of surface in a locality gets captured in these tiny patches. For example, RGB images from the Berkeley Segmentation Benchmark (BSB) \cite{Martin:2001} have been taken into consideration for the current study. In the case of a single RGB image, three dimensions exist in the colour map. These three form a feature set ($\mathcal{N} = 3$). \par
Several advantages arise with the use of this new hybrid model, namely: $\bullet$ Faster processing of the image on $1$D  compared to analysis of neighbourhood information per pixel. $\bullet$ Generation of homogeneous surfaces of different sizes that are bounded probabilistically, by inequality. $\bullet$ Leakage problem gets avoided due to conservative nature of inequality. $\bullet$ Reduction in problem sample size by a factor $\epsilon$ (Chebyshev parameter). $\bullet$ The density of image gets investigated via the multivariate formulation of the inequality. $\bullet$ The generalized hybrid model adapts to multivariate information while traversing on a $1$D fixed curve. \par
Hitherto, a brief description of the state of the work has been covered. In section \ref{sec:theory}, the theoretical aspect of the method is dealt in a greater detail with a toy image example. Experiments section \ref{sec:experiments} deals with the empirical evaluations conducted on BSB. Lastly, the conclusion follows in section \ref{sec:conclusion}. \par
\begin{figure}[!t]
\begin{center}
\includegraphics[width=8.5cm,height=6.5cm]{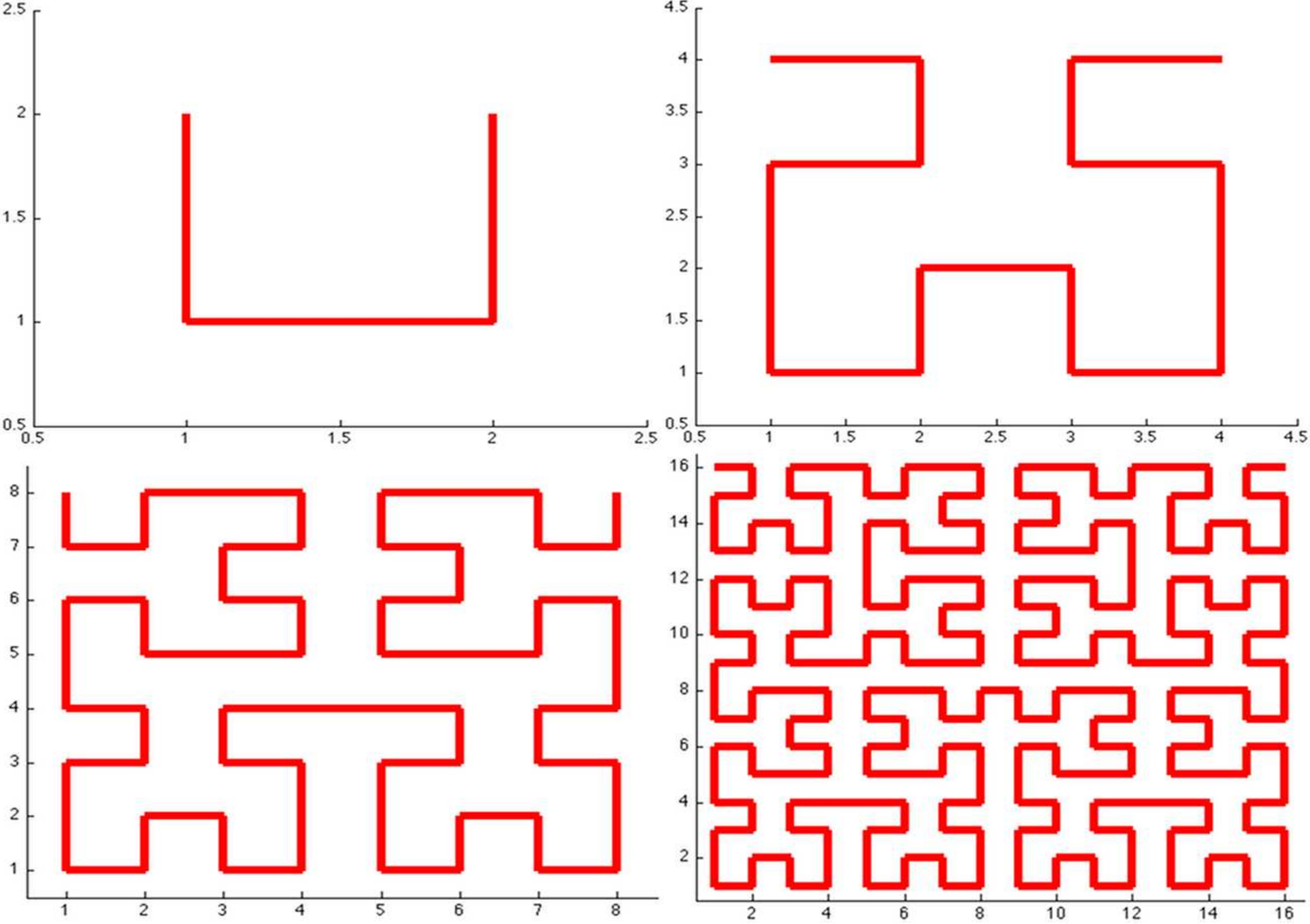}
\caption{Hilbert Space Filling Curve for grid of size $2$ (top left),
  $4$ (top right), $8$ (bottom left) and $16$ (bottom right).} 
\label{fig:hlbcrv}
\end{center}
\end{figure}
\begin{figure}[h]
\begin{center}
\includegraphics[width=8.5cm,height=10cm]{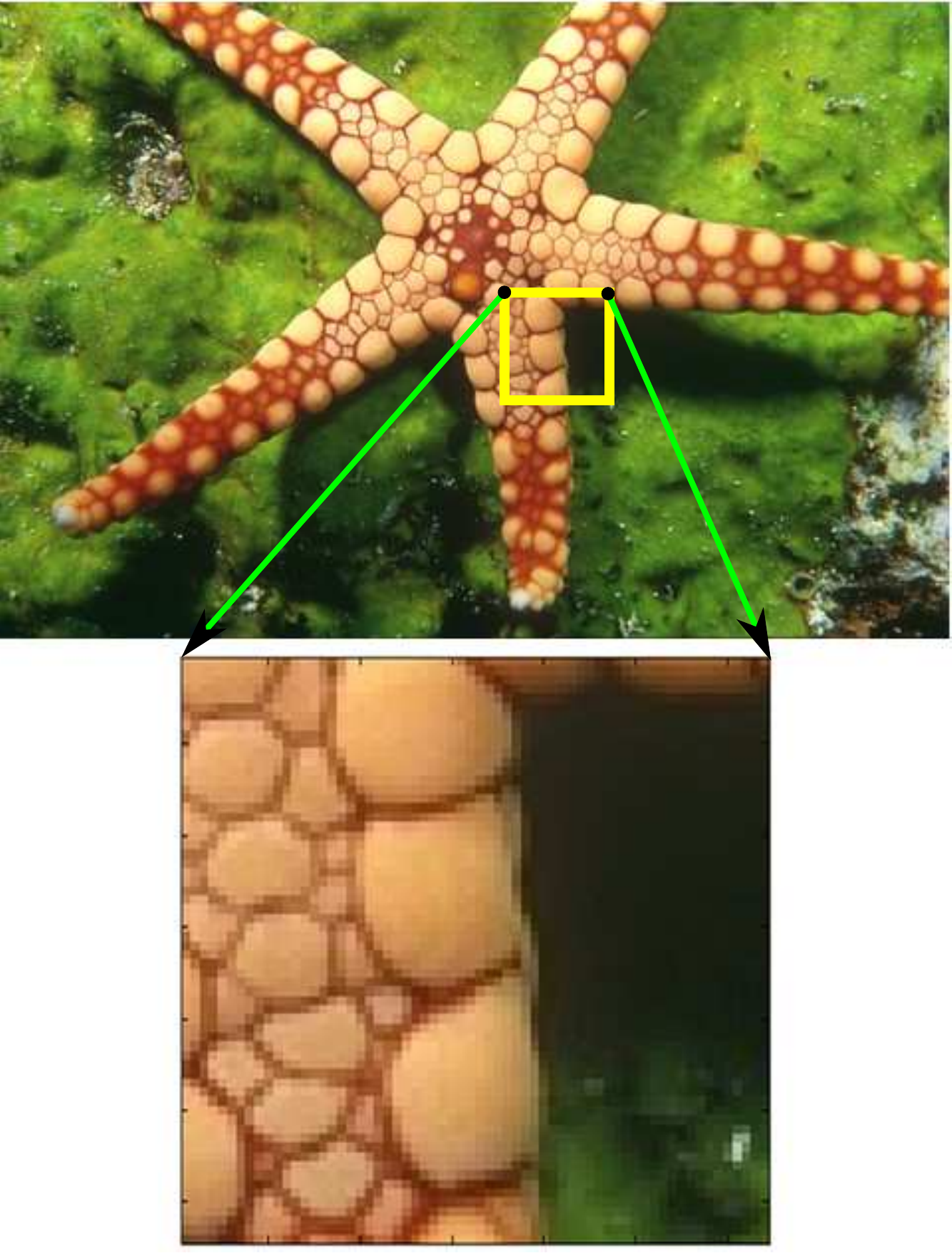}
\caption{Starfish image from \cite{Martin:2001} and the $64 \times 64$
  block under consideration.} 
\label{fig:starfish}
\end{center}
\end{figure}

\section{Theoretical Perspective}\label{sec:theory}
\subsection{Hilbert Space Filling Curve}\label{sec:hsfc}
The space filling curves form an important subject as it helps in transforming a multidimensional dataset into a linear format. This comes at a price of losing some amount of information, but the merits of preserving the local properties while transforming the objects in multi dimension to single dimension out weigh the incurred cost. The HSFC is a fractal filling curve proposed by \cite{Hilbert:1891} which fills the space of $2$D place in a continuous manner. Analytical results found in \cite{Moon:2001}, \cite{Jagadish:1990} and \cite{Gotsman:1996} prove the optimality of results obtained while using HSFC. In the current formulation a matlab implementation of \cite{Lau:1998} is used to generate the HSFC for $2$D matrices. \par
Figure \ref{fig:hlbcrv} shows the space filling curve for the grids of size $2$, $4$, $8$ and $16$ respectively. Note that the curve covers each and every point on the integer grid once while taking into account the local properties. It is not that the HSFC does not work for rectangular matrices, but the analysis of cluster preserving properties of the same becomes asymptotic in nature rather than being exact, as has been proved in \cite{Moon:2001}. Given an $\mathcal{N}$D image, the HSFC is generated which remains invariant of the same. \par
\begin{figure}[!t]
\begin{center}
\includegraphics[width=6cm,height=6cm]{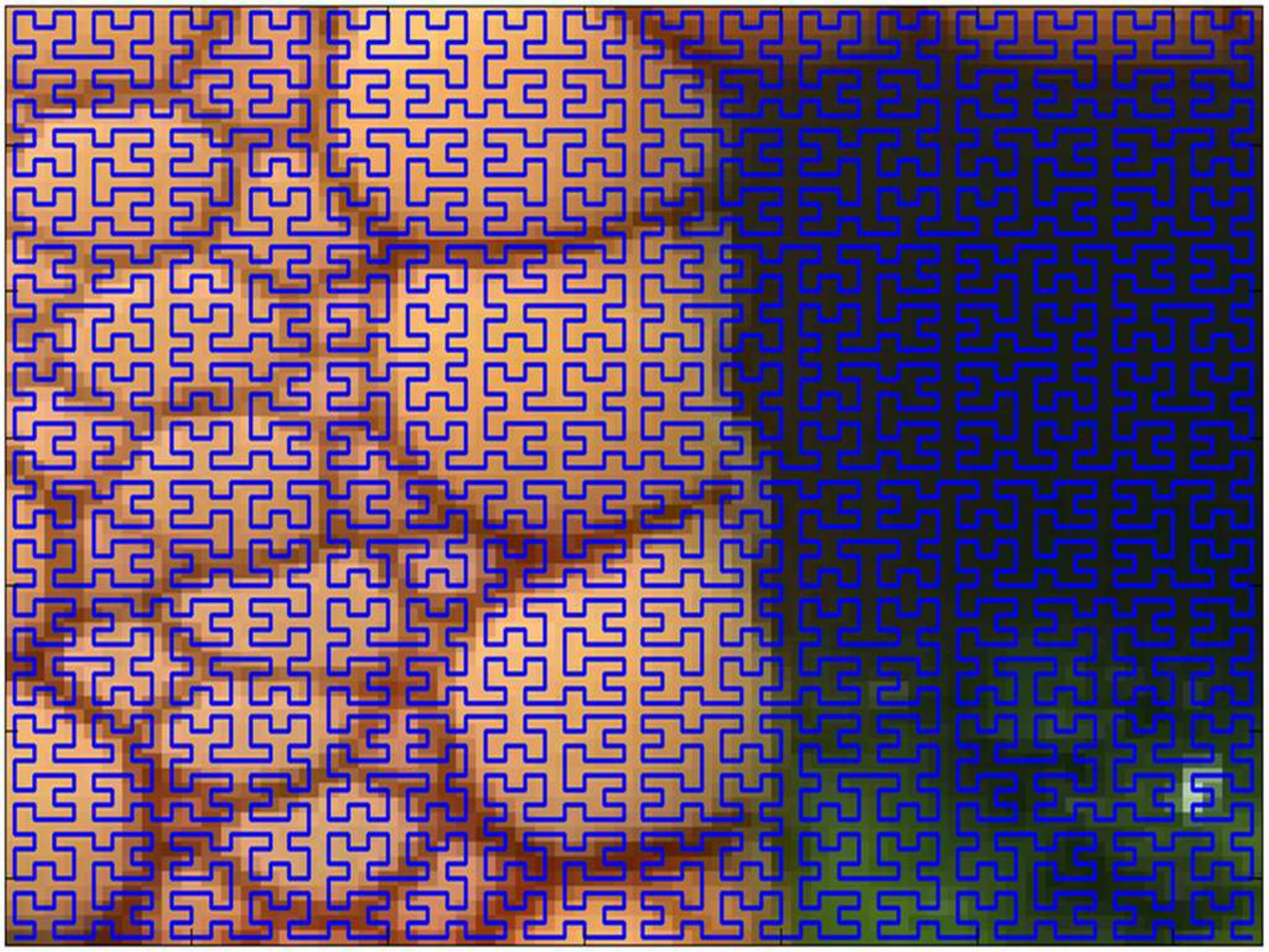}
\caption{HSFC on the patch of $64 \times 64$.} 
\label{fig:hlbstarfish}
\end{center}
\end{figure}

\subsection{Multivariate Chebyshev Inequality}\label{sec:mci}
Let $X$ be a stochastic variable in $\mathcal{N}$ dimensions with a mean $E[X]$. Further, $\Sigma$ be the covariance matrix of all observations, each containing $\mathcal{N}$ features and $\epsilon \in \mathcal{R}$, then the multivariate TChebyshev Inequality in \cite{Chen:2007} states that: 
\begin{eqnarray}
\mathcal{P} \{(X - E[X])^{T} \Sigma^{-1} (X - E[X]) \geq \epsilon\} & \leq &
\frac{\mathcal{N}}{\epsilon} \nonumber\\
\mathcal{P}\{(X - E[X])^{T} \Sigma^{-1} (X - E[X]) < \epsilon\} & \geq &
1 - \frac{\mathcal{N}}{\epsilon} \nonumber\\
\label{equ:ti}
\end{eqnarray}
i.e. the probability of the spread of the value of $X$ around the sample mean $E[X]$ being greater than $\epsilon$, is less than $\mathcal{N}/\epsilon$. There is a minor variation for the univariate case stating that the probability of the spread of the value of $x$ around the mean $\mu$ being greater than $\epsilon\sigma$ is less than $1/\epsilon^{2}$. Apart from the minor difference, both formulations convey the same message about the probabilistic bound imposed when a random vector or number $X$ lies outside the mean of the sample by a value of $\epsilon$. \par
In a broader perspective, the goal being to demarcate regions based on surfaces, two questions need to be addressed regarding the decomposition of image. $\bullet$ Which two pixels or their corresponding $\mathcal{N}$D vectors be selected to initialize a surface depicting \emph{near uniform behaviour}? $\bullet$ What should be the size of such a restricted surface? \par 

\subsection{Initializing Surface} 
The solution to first question would help in \emph{initializing a surface}. A pair of vectors in $\mathcal{N}$D will swap a flat plane with an angle subtended in between the two vectors. Given that the dot product exists in the higher dimensional plane, the cosine of the angle $\theta$ between the two vectors would suggest the degree of closeness between them. If $\vec{a}$ and $\vec{b}$ are two such vectors in $\mathcal{N}$D, then the degree of closeness is given by:
\begin{equation}
cosine \theta = \frac{<\vec{a}, \vec{b}>}{||\vec{a}||_{2}\times||\vec{b}||_{2}}
\label{equ:dot}
\end{equation}
were, $<\vec{a}, \vec{b}>$ is the dot product and the denominator contains the $2$-norm terms of both vectors. It is well known that the absolute value of  $cosine \theta$ tends to $1$ ($0$) as vectors tend to be nearly parallel (perpendicular). Let $\vec{u}$, $\vec{v}$ and $\vec{w}$ be three consecutive pixels on the HSFC. If the $cosine$ of the angle between $\vec{u}$ and $\vec{v}$ evaluates to an absolute value greater than a nearness threshold $Npar$ (say $0.95$) then the pair is considered as a valid surface. Note that $Npar$ is the nearness parameter which is used as a threshold to decide the degree of closeness of two pixels for forming a surface. If not, then $\vec{u}$ is left as a single point in $\mathcal{N}$D and the closeness criterion is checked for $\vec{v}$ and $\vec{w}$ (and the process is repeated). \par

\subsection{Size Of Surface}
Solving the second question shall define the \emph{range of the surface}. Once the start and end points (say $\vec{v}$ and $\vec{w}$ respectively) of a valid surface have been set, the size of the surface has to be determined. The size of the surface would constitute all points that contribute towards uniform surface behaviour in $\mathcal{N}$D. This degree of uniformity is controlled via the Chebyshev's Inequality. The idea is executed as follows: The next consecutive point (say $\vec{t}$) after $\vec{w}$ is considered for surface extent analysis. If the spread of surface point $\vec{t}$ from $E(surface)$ the mean of the \emph{existing surface} $[\vec{v}, \vec{w}]$, factored by the covariance matrix, is below $\epsilon$, then $\vec{t}$ is considered as a part of the surface and marked as a new ending point of the surface. Using Chebyshev's Inequality, it boils down to:
\begin{eqnarray}
\mathcal{P} \{(\vec{t} - E([\vec{v},\vec{w}])^{T} \Sigma^{-1} (\vec{t} -
E([\vec{v},\vec{w}]) \geq \epsilon\} & \leq &
\frac{\mathcal{N}}{\epsilon} \nonumber\\
\mathcal{P} \{(\vec{t} - E([\vec{v},\vec{w}])^{T} \Sigma^{-1} (\vec{t} -
E([\vec{v},\vec{w}]) < \epsilon\} & \geq & 1 -
\frac{\mathcal{N}}{\epsilon} \nonumber\\
\label{equ:ti_imp}
\end{eqnarray}
were $\Sigma$ is the covariance matrix between the $\mathcal{N}$D vectors constituting the initial surface. Satisfaction of this criterion leads to extension of the size of initial surface by one more point i.e. $\vec{t}$. The surface now constitutes $[\vec{v}, \vec{w}, \vec{t}]$, with $\vec{v}$ and $\vec{t}$ as start and end marker points. If not, the size of the surface remains as it is and a fresh start is made starting with $\vec{t}$ and the next consecutive point on the HSFC. The satisfaction of the inequality also gives a lower probabilistic bound on size of surface by a value of $1 - (\mathcal{N}/\epsilon)$, if the second version of the Chebyshev formula is under consideration. \par
The above formulation implies that when a homogeneous patch is encountered, then the a new point does not deviate much from the initial surface. Thus the size of the surface grows smoothly. For a highly irregular patch, the surface size may be very restricted due to high variation of a pixel from the surface it is being tested in the vicinity. Figure \ref{fig:starfish} shows the tiny patch ($64 \times 64$) of starfish under consideration and figure \ref{fig:hlbstarfish} shows the HSFC generated over the area of the image. \par
The tiny restricted surfaces generated using the multivariate and the univariate formulation of the Tchebyshev inequality for $\epsilon = 3$ are shown in figure \ref{fig:multi_e3} and \ref{fig:uni_e3}. Note how the surfaces differ due to the multivariate and univariate formulation of the inequality. The former takes into account the entire $\mathcal{N}$D vectors in tandem to compute the covariance or the texture interaction and the mean, while the later computes the inequality separately for each and every dimension. The different colours just indicate the different surfaces and has nothing to do with clustering at this stage. The potentiality of the method gets highlighted due to the bounded surfaces that have been obtained after image decomposition. This boundedness is checked via the parameter which determines the degree of intricacy to which the texture interaction is to be taken into account. It should be noted that in the univariate case the decision is made on the majority vote after $\mathcal{N}$ evaluations of the inequality on the space filling curve for a single pixel. What paves way is that the univariate formulation does not capture the  interaction which later leads to low grade
segmentation results compared to that given by the  formulation. These differences are apparent in the figures mentioned at the starting of the paragraph. Figure \ref{fig:ds_multi} and  \ref{fig:ds_uni} shows the image patch decomposed into surface patches for different $\epsilon$ values. \par
\begin{figure}[!t]
\begin{center}
\includegraphics[width=6cm,height=6cm]{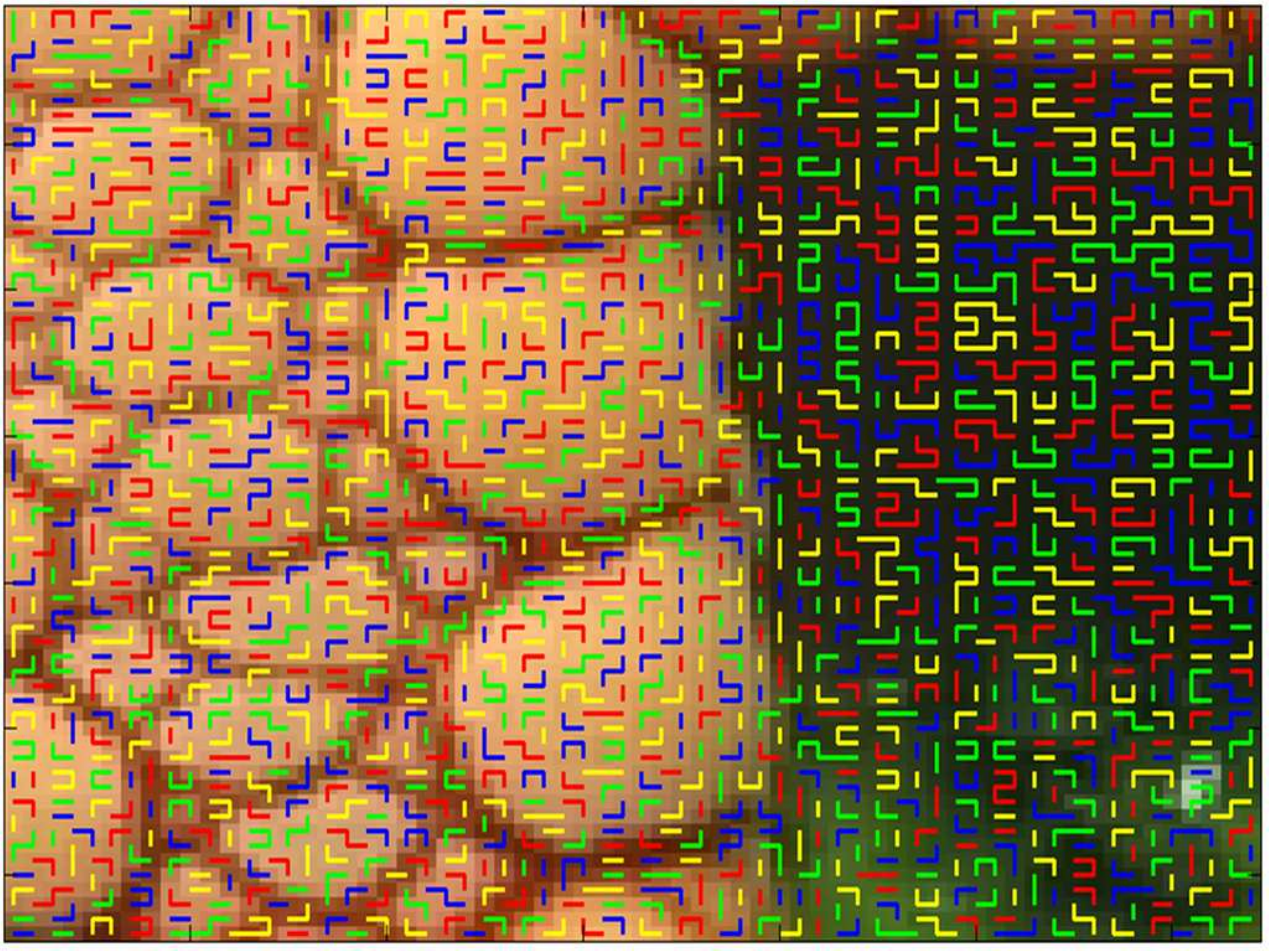}
\caption{Patch of $64 \times 64$ image of starfish from \cite{Martin:2001}, decomposed using the multivariate formulation of Chebyshev inequality with $\epsilon = 3$ and $Npar = 0.95$. A coloured line shows the pixels associated to a single bounded surface on the HSFC.} 
\label{fig:multi_e3}
\end{center}
\end{figure}
\begin{figure}[!t]
\begin{center}
\includegraphics[width=6cm,height=6cm]{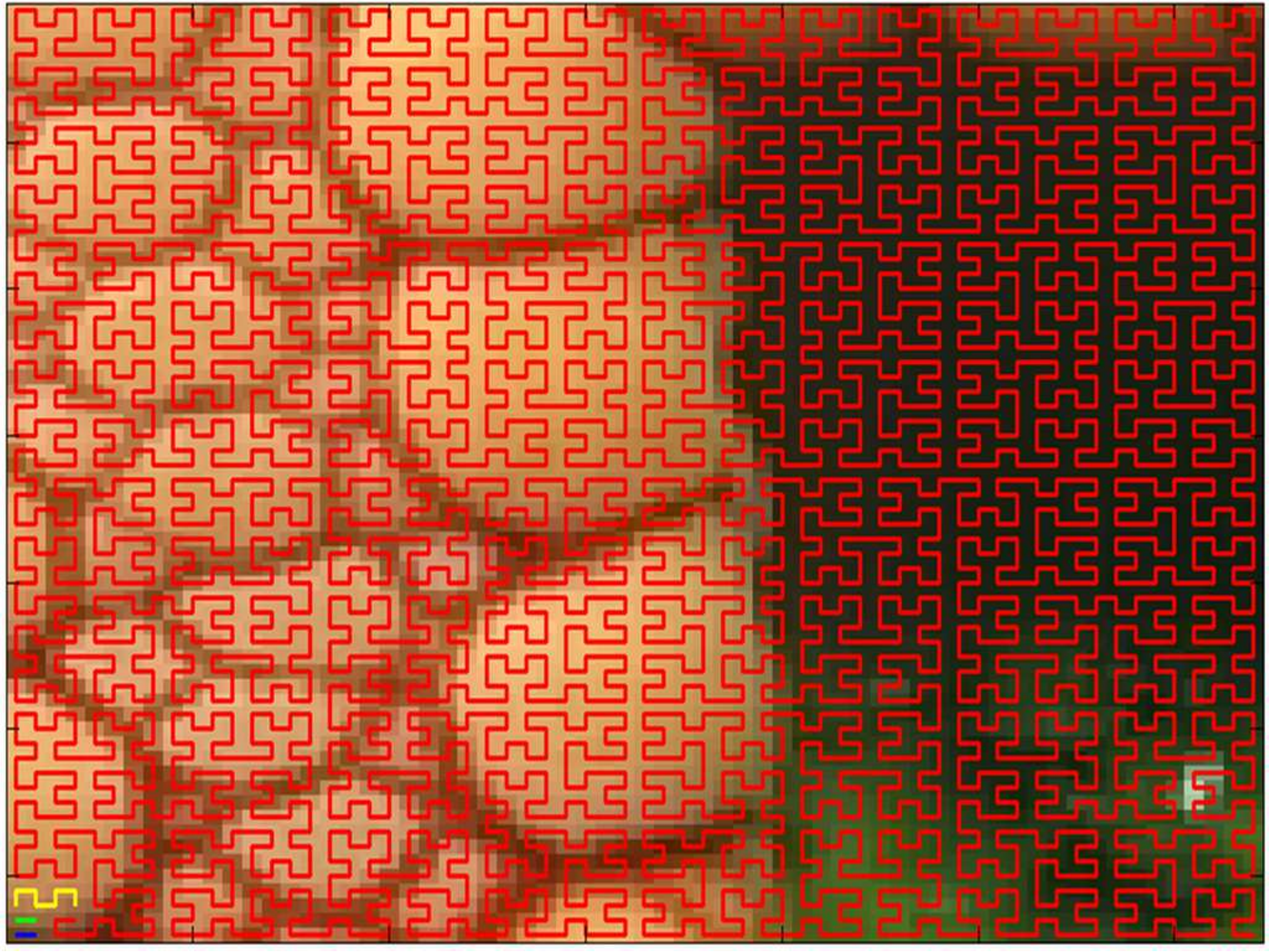}
\caption{Patch of $64 \times 64$ decomposed using the univariate formulation of Tchebyshev inequality with $\epsilon = 3$ and $Npar = 0.95$. A coloured line shows the pixels associated to a single bounded surface on the HSFC. Note that size of surface is decided based on voting across $\mathcal{N}$ evaluations.} 
\label{fig:uni_e3}
\end{center}
\end{figure}

\subsection{Implications}\label{sec:implications}
The inequality being a criterion, the probability associated with the same gives a belief based bound on the satisfaction of the criterion. This gives rise to certain simple implications as follows. Let $\mathcal{D}$ be a \emph{decomposition} which is equivalent to $(X_{t} - E[X])^{T} \Sigma^{-1} (X_{t} - E[X])$. Then: 
\begin{figure}[!t]
\begin{center}
\includegraphics[width=8.5cm,height=6cm]{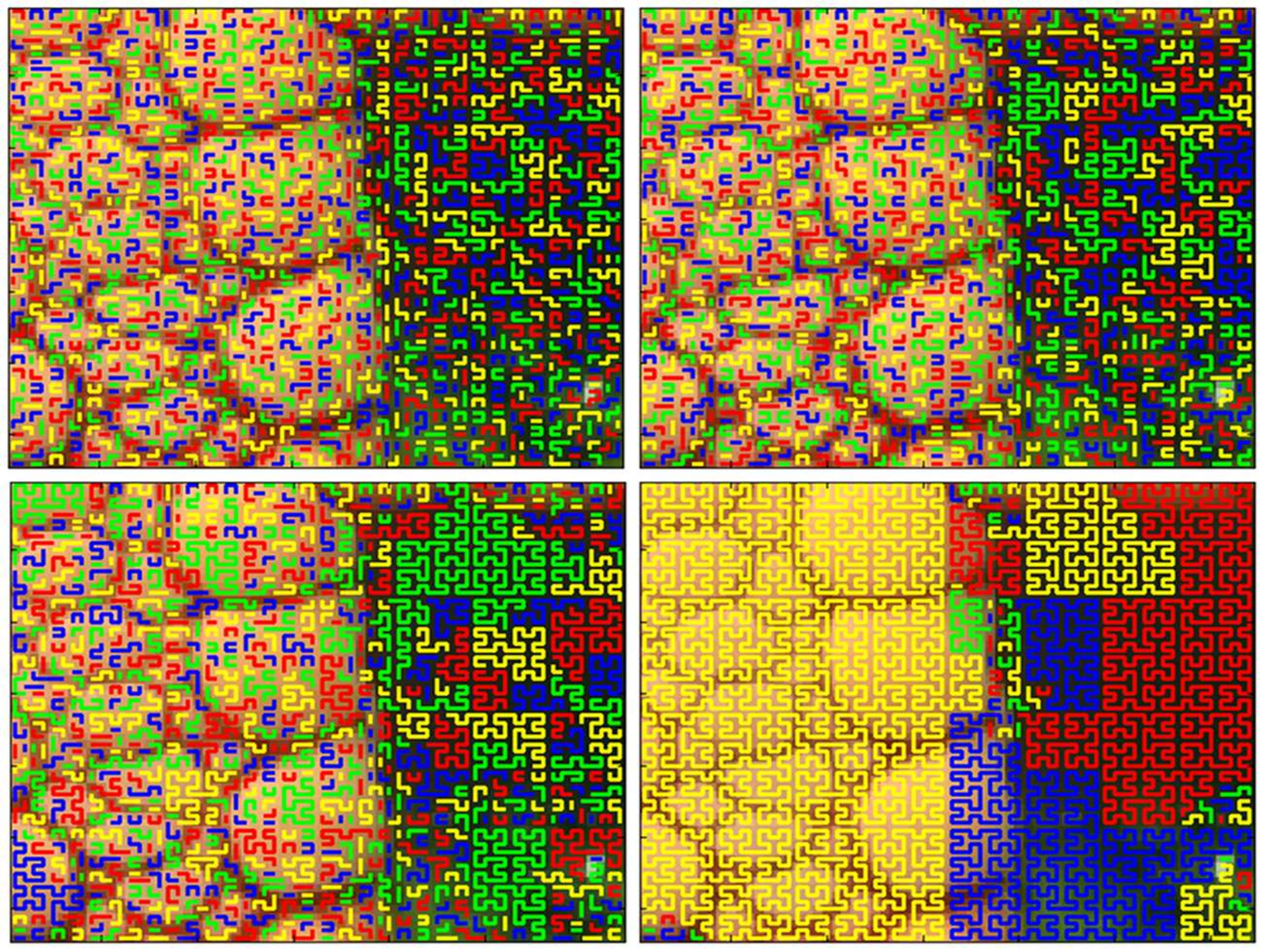}
\caption{Decomposed surfaces using multivariate inequality with
  $Npar = 0.95$ and for $\epsilon$ equal to $4$ (top left),
  $8$ (top right),$16$ (bottom left) and $32$ (bottom right).} 
\label{fig:ds_multi}
\end{center}
\end{figure}
\begin{figure}[!t]
\begin{center}
\includegraphics[width=8.5cm,height=6cm]{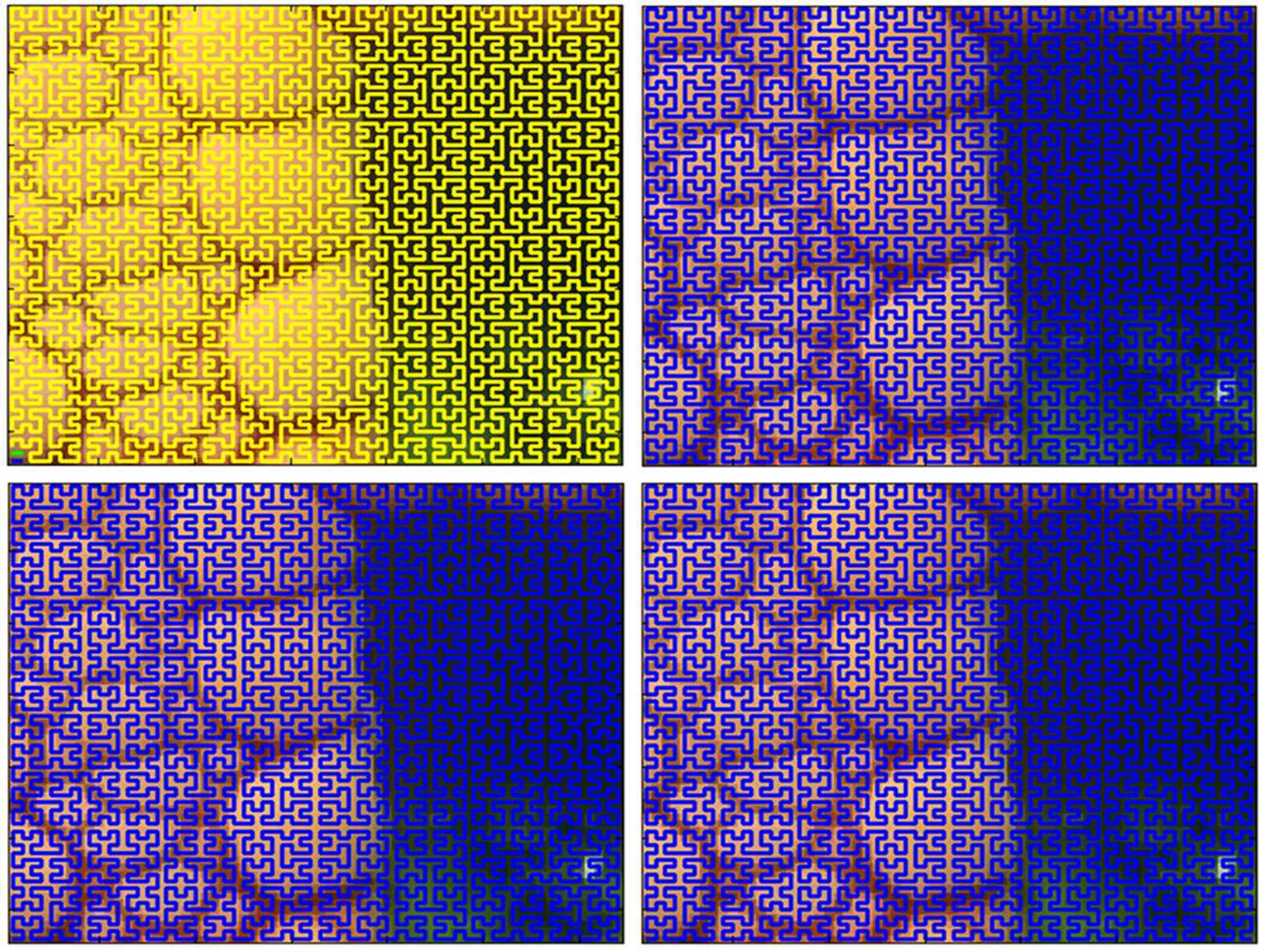}
\caption{Decomposed surfaces using univariate inequality with
  $Npar = 0.95$ and for $\epsilon$ equal to $4$ (top left),
  $8$ (top right),$16$ (bottom left) and $32$ (bottom right).} 
\label{fig:ds_uni}
\end{center}
\end{figure}
\begin{lemma}
Decompositions $\mathcal{D}$ are bounded by lower probability bound of $1 - (\mathcal{N}/\epsilon)$ given that $\epsilon \geq \mathcal{N}$. \label{lem:01}
\end{lemma}
Not only does it hold true for $\epsilon \geq \mathcal{N}$ but also for $\epsilon < \mathcal{N}$. But the probability being greater than a negative value is always true and thus $\epsilon = \mathcal{N}$ forms the lower bound. As $\epsilon \rightarrow \infty$, $\mathcal{P}(\mathcal{D} < \epsilon) \rightarrow 1$. \par
\begin{lemma}
The value of $\epsilon$ reduces the size of the sample from
$\mathcal{M}$ to an upper bound of $\mathcal{M}/\epsilon$ probabilistically with a lower bound of $1 - (\mathcal{N}/\epsilon)$. Here $\mathcal{M}$ is the number of pixels in a $2$D matrix. \label{lem:02}
\end{lemma}
This holds true as the image is decomposed into surfaces which are probabilistically bounded via the Chebyshev inequality. This decomposition leads to reduction in the sample size by a factor of $\epsilon$ while retaining the information content, kudos to the space filling curve traversal. \par
\begin{lemma}
As $\epsilon \rightarrow \mathcal{N}$ the lower probability bound drops to zero, implying large number of small decompositions $\mathcal{D}$ can be achieved. (Vice versa for $\epsilon \rightarrow \infty$) \label{lem:03}
\end{lemma}
This gives an insight into the degree to which the image can be decomposed. Where finner details are of import, one may use values of $\epsilon$ tending to $\mathcal{N}$ and vice versa. \par
\begin{theorem}
Let image $\mathcal{I}$ contain $\mathcal{M}$ pixels, with each pixel having $\mathcal{N}$ features . If $\mathcal{I}$ can be decomposed in $\ell = \mathcal{M}/\epsilon$ bounded surfaces via the proposed hybrid model, then in case of the decompositions having equally likely probabilities: (a) $\mathcal{M}$ = $\epsilon^{2}/(\epsilon - \mathcal{N})$ and (b) $\epsilon \in$ open interval ($\mathcal{N},\mathcal{M}$). \label{thm:01}
\end{theorem}
\begin{proof}
Since the $\mathcal{I}$ can be decomposed into $\ell$ bounded surfaces, it is known that the decompositions are \emph{disjoint} sets. Let $\mathcal{D}_{j}$ (for $j \in \{1,\ell\}$) be such decompositions. Then $\mathcal{I}$ = $\bigcup_{j = 1}^{\ell}\mathcal{D}_{j}$. Considering $\mathcal{I}$ as the universal set, we get: $\mathcal{P}(\mathcal{I}) = \mathcal{P}(\bigcup_{j=1}^{\ell} \mathcal{D}_{j}) = \sum_{j=1}^{\ell} \mathcal{P}(\mathcal{D}_{j})$. From lemma \ref{lem:01}, it is known that $\mathcal{P}(\mathcal{D}_{j} < \epsilon) \geq 1 - (\mathcal{N}/\epsilon)$. In the case that the decompositions are equally likely, the lowest probability for each $\mathcal{D}_{j}$ evaluates to $1 - (\mathcal{N}/\epsilon)$.\par
Thus, $\sum_{j = 1}^{\ell}\mathcal{P}(\mathcal{D}_{j})$ = $\ell(1 - \mathcal{N}/\epsilon)$ = $(\mathcal{M}/\epsilon)(1 - \mathcal{N}/\epsilon)$. Since $\mathcal{P}(\mathcal{I}) = 1$ and $\mathcal{P}(\mathcal{I})$ = $\sum_{j=1}^{\ell} \mathcal{P}(\mathcal{D}_{j})$, it implies that $(\mathcal{M}/\epsilon)(1 - \mathcal{N}/\epsilon)$ = $1$. Simplifying the foregoing formulation leads to $\mathcal{M} = \epsilon^{2}/(\epsilon - \mathcal{N})$. This proves the part (a) of the theorem. \par
From part (a), it can be clearly seen that $\epsilon - \mathcal{N} \neq 0$, lest part (a) would be invalid. Thus $\epsilon \neq \mathcal{N}$. Similarly, if $\epsilon = \mathcal{M}$, then on simplification of part (a) evaluates to $\mathcal{N} = 0$. Again this cannot be the case as an image will have atleast one feature in terms of intensity. Thus $\epsilon \neq \mathcal{M}$. Let $k > 1$, then for $\epsilon = \mathcal{N}/k$ part (a) shows that the ratio of $\mathcal{M}/\mathcal{N}$ is a negative quantity. This cannot be possible as both numbers in the ratio are positive numbers. Thus $\epsilon > \mathcal{N}$. Again, for the same condition of $k$, if $\epsilon = k\mathcal{M}$, part (a) on evaluation shows the ratio of $\mathcal{M}/\mathcal{N}$ to be a negative quantity. This again is a contradiction given the state of $\mathcal{M}$ and $\mathcal{N}$. Thus $\epsilon < \mathcal{M}$. Thus $\epsilon$ is strictly bounded in the open interval ($\mathcal{N},\mathcal{M}$). This finishes proof for part (b). \qed
\end{proof}
\begin{theorem}
If an image $\mathcal{I}$ is decomposed into bounded surfaces s.t. the latter have equally likely probabilities, then theoretically the maximum value of both $\epsilon$ (Chebyshev parameter) and $\mathcal{N}$ (number of feature dimensionality of $\mathcal{I}$) is of the order of $\sqrt{\mathcal{M}}$. \label{thm:02}
\end{theorem}
\begin{proof}
From theorem \ref{thm:01}, it is known that $\mathcal{M} = \epsilon^{2}/(\epsilon - \mathcal{N})$, when the probabilities of the decomposed surfaces are equally likely. It is also known that $\mathcal{N} < \epsilon < \mathcal{M}$. Let $k > 1$ and $\epsilon$ decrease harmonically via $\epsilon = \mathcal{M}/k$ for $k$ from $2$ onwards. Then simplifying part (a) of theorem \ref{thm:01} gives $\mathcal{M}/\mathcal{N} = k^{2}/(k - 1)$. \par
Now, if $\epsilon = \mathcal{M}/\mathcal{N}$, then $\mathcal{M}/k = k^{2}/(k-1)$. The fraction on the right hand side can be segregated into complete and partial fractions as $k + k/(k-1)$. It is known that $k+1 < k + k/(k-1) < k+2$. Then $k+1 < \mathcal{M}/k < k+2$. Equating for both inequalities around $\mathcal{M}/k$, the value of $k$ lies between $\sqrt(1+\mathcal{M}) + 1$ and $(\sqrt(1 + 4\mathcal{M}) - 1)/2$. Taking the order of maximum value of $k$ as $\sqrt(\mathcal{M})$, $\epsilon$ = $\mathcal{M}/k$ = $\sqrt(\mathcal{M})$ and $\mathcal{N}$ = $k$ = $\sqrt(\mathcal{M})$.\qed
\end{proof}
\begin{figure}[!t]
\begin{center}
\includegraphics[width=8.5cm,height=6cm]{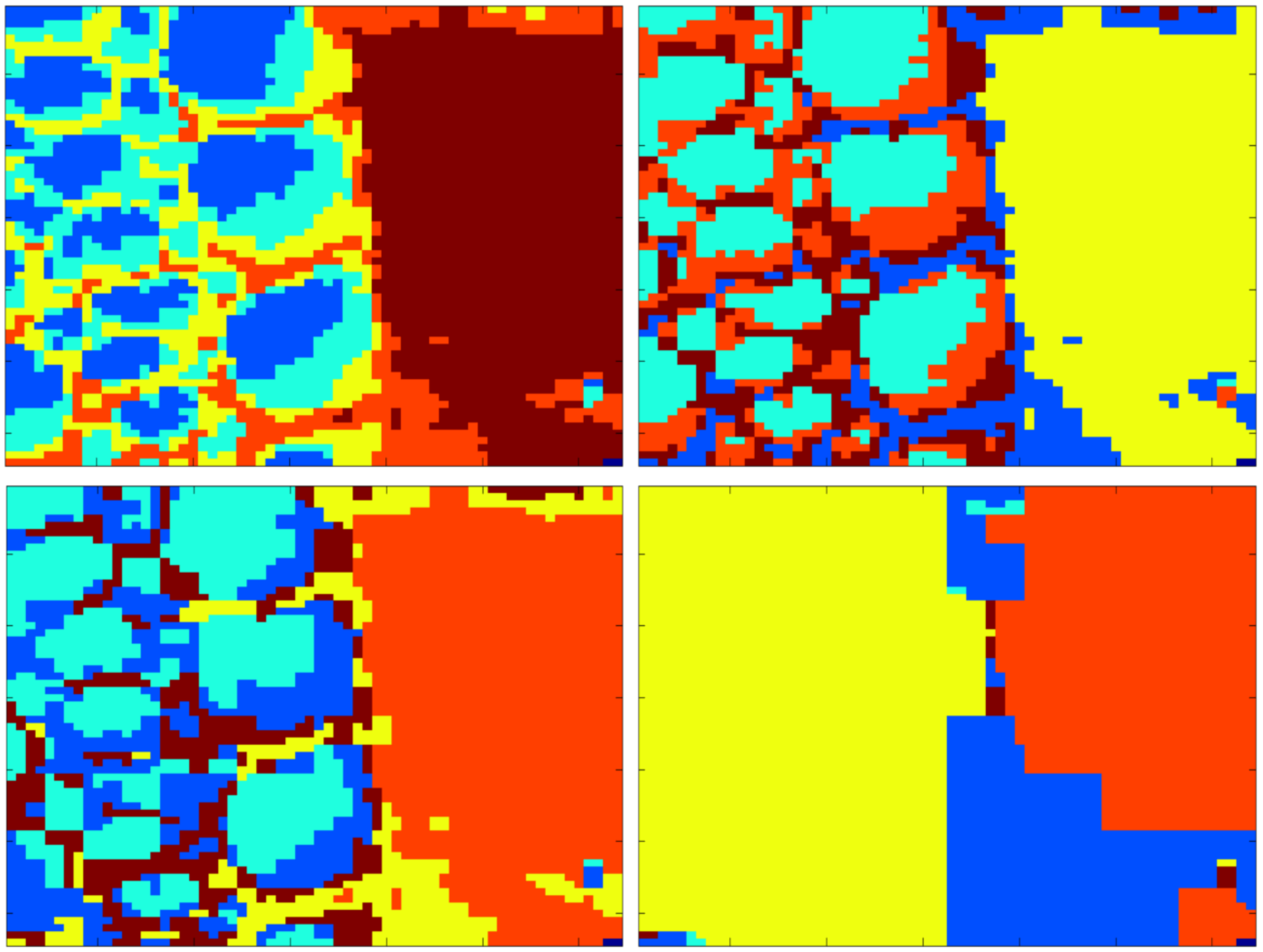}
\caption{Image with $5$ clusters using $k$-means (cityblock distance, $100$ replicates and $1000$ iterations) on decomposed surfaces generated via multivariate inequality with $Npar = 0.95$ and for $\epsilon$ equal to $4$ (top left), $8$ (top right),$16$ (bottom
  left) and $32$ (bottom right).}
\label{fig:c5_multi}
\end{center}
\end{figure}
Theorem \ref{thm:01} shows how the sample size $\mathcal{M}$ is related to the dimensionality of feature space $\mathcal{N}$ via the Chebyshev parameter $\epsilon$. Implicitly, it also states that the sample size $\mathcal{M}$ must be greater than $\mathcal{N}$ and the the value of $\epsilon$ lies in the open interval $(\mathcal{N},\mathcal{M})$. These tight bounds in an idealistic case show that the model is effective in a theoretical sense. The second theorem builds on the first and shows that the maximum theoretical value of the parameter $\epsilon$ is of the order of square root of the sample size $\mathcal{M}$ and so is the dimensionality of the feature space $\mathcal{N}$. In an ideal case of equal likelihood, this maximum value gives an upper bound on $\epsilon$ as well as $\mathcal{N}$, such that the decompositions are uniformally spread in the higer dimensional space. \par

\subsection{Clustering On Bounded surfaces}\label{sec:clust}
Once the image has been decomposed into bounded surfaces, segmentation of the image is done on the average value of surfaces via $k$-means algorithm in Matlab with a certain number of pre-defined clusters. The cityblock distance is used as a metric for the kmeans and the number of replicates is of the order of $100$ with $1000$ iterations for the $k$-means. The reason for using a high number repilcates and iterations is to avoid getting stuck in local solutions. The average values are computed by taking the mean of the $\mathcal{N}$D points that constitute the bounded surfaces. These values are considered to be robust as the surfacess themselves are bounded on the first place probabilistically taking into account the variability in intensity behaviour. \par
\begin{figure}[!t]
\begin{center}
\includegraphics[width=10cm,height=7cm]{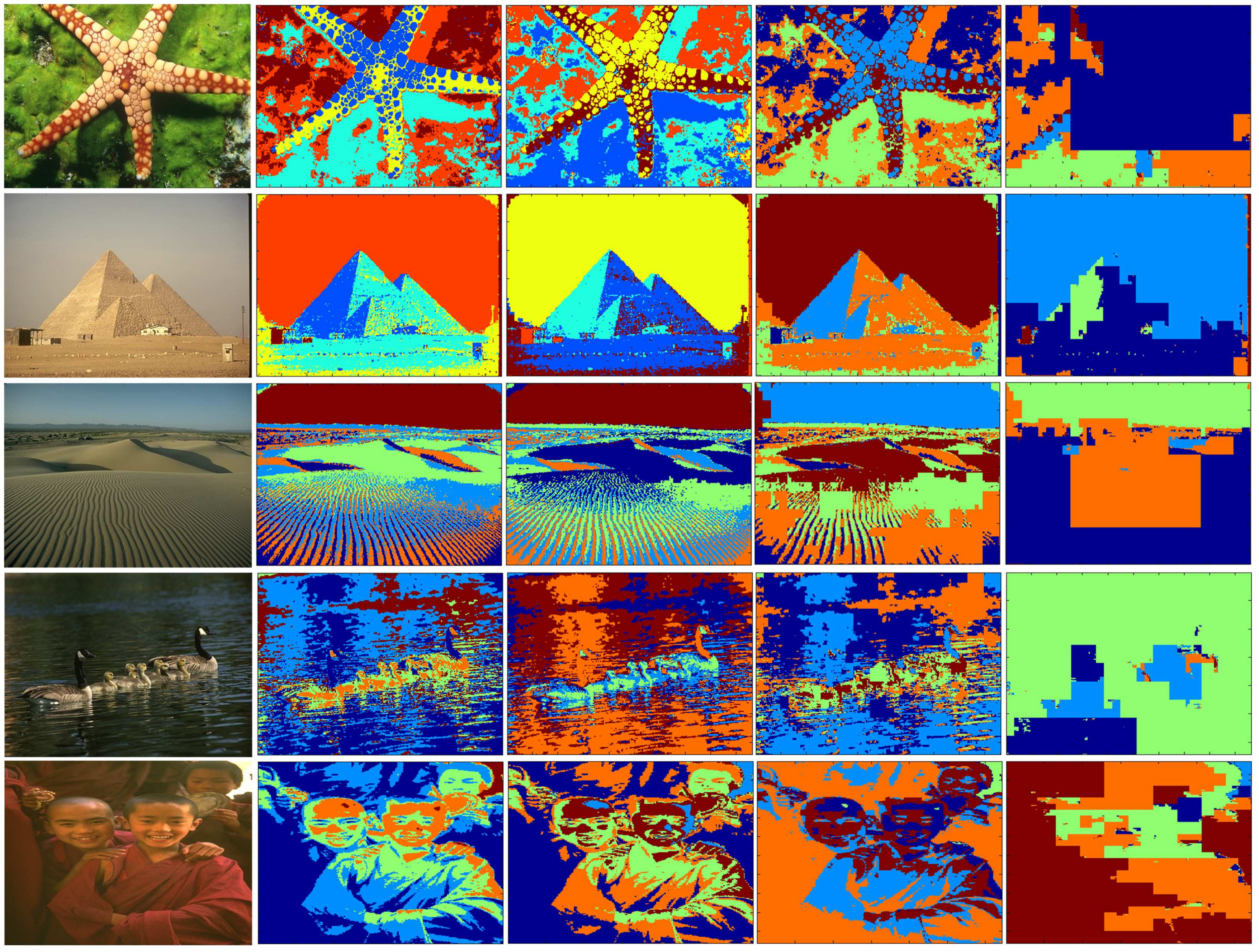}
\caption{Images from \cite{Martin:2001} segmented for
  $Npar = 0.95$ and for $\epsilon$ equal to $4$,
  $8$,$16$ and $32$, from left but one to right with number of clusters $5$.} 
\label{fig:berk_clust_a}
\end{center}
\end{figure}
\begin{figure}[!t]
\begin{center}
\includegraphics[width=10cm,height=7cm]{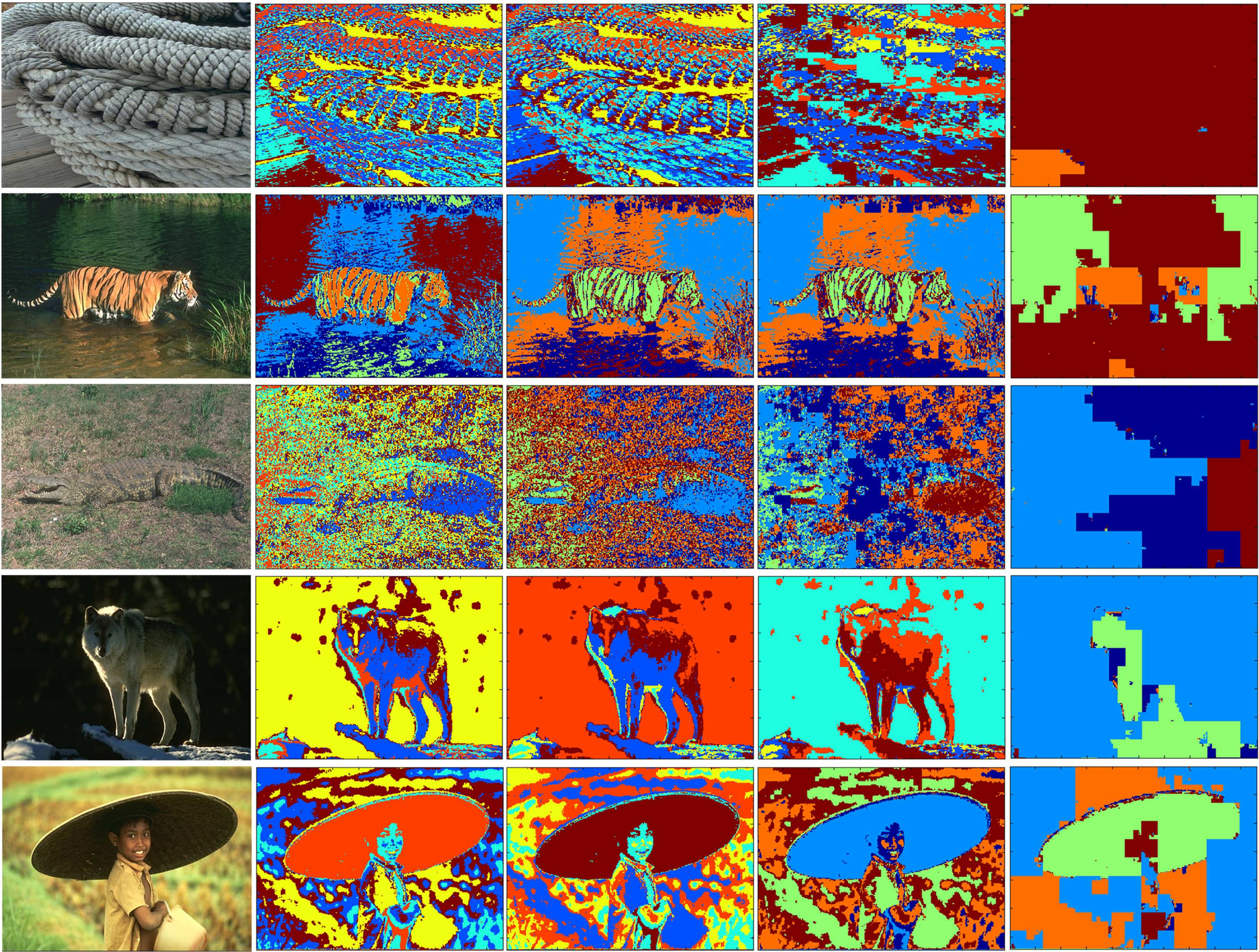}
\caption{Images from \cite{Martin:2001} segmented for
  $Npar = 0.95$ and for $\epsilon$ equal to $4$,
  $8$,$16$ and $32$, from left but one to right with number of clusters $5$.} 
\label{fig:berk_clust_b}
\end{center}
\end{figure}
Figure \ref{fig:c5_multi} shows the corresponding result for different values of $\epsilon$. The figure shows that the quality of segmentation degrades as the value of $\epsilon$ increases, which increases the surface size. For example, the groves on the starfish are captured in greater details for $\epsilon = 4$ than for higher values of the same. A few sample images from \cite{Martin:2001} for which segmented images were generated over different values of $\epsilon \in \{4,8,16,32\}$ have been presented in figures \ref{fig:berk_clust_a} and \ref{fig:berk_clust_b}. The number of clusters was predefined to be $5$. With increasing size of $\epsilon$ values, the amount of decompositions reduce which later affect the quality of the segmentation. This is apparent in the figures as one moves from left to right.The clustering on these bounded surfaces is not only good but also less time consuming as the segmentation are done on a reduced sample size while still retaining the crucial pieces of information. Not that the solution is the best, but results tend to be good quality at first sight.

\section{Decomposition Algorithm}\label{sec:algorithm}
Algorithm \ref{alg:final_a} and its continuity in \ref{alg:final_b} shows the implementation for the decomposing the image into probabilistically bounded surfaces based on space filling curve traversal. The depicted version is for multivariate formulation of the Tchebyshev inequality. Minor change in the form of univariate formulation used separately with each of the $\mathcal{N}$ dimensions would lead to univariate version. \par
Note that the output of the decomposition algorithm is a list of bounded surfaces. Many features could be developed but in this manuscript mean values of all the $\mathcal{N}$D points constituting a surface is taken as a feature vector. This is because the inequality measures the degree of homogeneity of density interaction in $\mathcal{N}$D. \par
\begin{algorithm}[pth]
\caption{Tchebyshev Surfaces}\label{alg:final_a}
\begin{algorithmic}[1]
\Procedure{TchSurf}{$img$, $\epsilon$, $Npar$,
  $formulation$}
\State [$nrows$, $ncols$] $\gets$ size($img$)
\State [$R$,$S$] $\gets$ hlbrtcrv($nrows$, $ncols$) \Comment{HSFC
  coords}
\Statex Initialize Variables
\State $marked_{vertex} \gets$ [] \Comment{Marked pixels per surface}
\State $surface_{list} \gets$ {} \Comment{List of surfaces}
\State $feature_{list} \gets$ {} \Comment{List of features per surface}
\State $surface_{no} \gets 0$ \Comment{Count of surfaces}
\Statex Generate bounded surfaces on HSFC
\State $s_{idx} \gets 1$ \Comment{Index of pt where the curve starts}
\While{$s_{idx} < len-1$} \Comment{$len$ is length of curve}
     \State $e_{idx} \gets s_{idx} + 1$
     \State [$rs$, $cs$] $\gets$ [$R$($s_{idx}$), $S$($s_{idx}$)]
     \State [$re$, $ce$] $\gets$ [$R$($e_{idx}$), $S$($e_{idx}$)]
     \State $pix_{info}$ $\gets$ struct(); \Comment{Initially two pixels}
     \State $pix_{info}.loc$ $\gets$ [$rs$, $cs$; $re$, $ce$]
     \Comment{store locations}
     \Statex Store intensity values per pixel
     \State $surface \gets$ [] 
     \State [$lenidx$, $cols$] $\gets$ size($pix_{info}.loc$)
     \For{$i = 1:lenidx$}
          \State [$r$, $c$] $\gets$ [$pix_{info}.loc(i, 1)$, $pix_{info}.loc(i, 2)$]
          \State $temp \gets$ []
          \For{$j = 1:\mathcal{N}$} \Comment{Features}
               \State $temp \gets$ [$temp$; $img$($c$, $r$, $j$)]
          \EndFor
          \State $surface \gets$ [$surface$, $temp$]
     \EndFor
     \Statex Compute nearness between two initial pixels
     \State $innerprod \gets$ dot($surface(:,1)$, $surface(:,2)$);
     \State $cosval \gets \frac{innerprod}{norm(surface(:,1),2) \times
       norm(surface(:,2),2)}$
     \Statex Check $cosine\theta$ greater than nearness param
     \If{$cosval \geq Npar$}
          \State $e_{idx} \gets e_{idx} + 1$
          \While{$e_{idx} \leq len$}
               \Statex Store next pixel on HSFC
               \State [$re$, $ce$] $\gets$ [$R$($e_{idx}$), $S$($e_{idx}$)]
               \State $pix_{loc}^{end} \gets$ [$re$, $ce$]
               \Statex Store intensity values for pixel
               \State $new_{pix} \gets$ []
               \For{$j = 1:\mathcal{N}$}
                    \State $new_{pix} \gets$ [$new_{pix}$; $img$($ce$, $re$, $j$)]
               \EndFor
\algstore{bkbreak}
\end{algorithmic}
\end{algorithm}
Another point to be aware of is the computation of the measure in multivariate Chebyshev inequality in equation \ref{equ:ti_imp}. Since it requires presence of inverse of the covariance matrix, one may run into problem of sparseness, or inappropriate dimensionality between the number of samples and features. To overcome these problems, the pseudo inverse was computed using the singular value decomposition implementation in matlab. Being fast and effective, the decomposition of the image into bounded surfaces works within a matter of seconds. \par
\begin{algorithm}[pth]
\caption{Tchebyshev Surfaces Continued}\label{alg:final_b}
\begin{algorithmic}[1]
\algrestore{bkbreak}
               \Statex Tchebychevs inequality for fixing length of surf
               \Statex NOTE - rows are dimension and cols are pixels in surface matrix
               \State $surfacemu \gets$ $\mu$($surface^{T}$)
               \State $surfacestd \gets$ std($surface^{T}$)
               \State $covR \gets$ cov($surface^{T}$)
               \State $dev \gets$ ($new_{pix} - surfacemu^{T}$)
               \Statex Compute inequality in equ \ref{equ:ti_imp}
               \State $criterion \gets$ $dev^{T}$ * pinv($covR$) * $dev$ 
               \State $decision \gets 0$
               \If{$criterion < \epsilon$}
                      \State $decision \gets decision + 1$
               \EndIf
               \If{$decision \geq 1$}
                     \State $surface \gets$ [$surface$, $new_{pix}$]
                     \State $pix_{info}.loc \gets$ [$pix_{info}.loc$; $pix_{loc}^{end}$]
                     \State $e_{idx} \gets e_{idx} + 1$
               \Else
                      \State $break$
               \EndIf
          \EndWhile
          \Statex Store the surfaces, associated pixels and features
          \State $surface_{no} \gets surface_{no} + 1$
          \State $surface_{list}${$surface_{no}$} $\gets pix_{info}.loc$
          \Statex Mean per dimension as feature for a surface
          \State $featval \gets$ []
          \For{$i = 1:\mathcal{N}$}
              \State $featval \gets$ [$featval$; $\mu$($surface(i,:)$)]
          \EndFor
          \State $feature_{list}\{surface_{no}\} \gets featval$
          \State $s_{idx} \gets e_{idx}$
     \Else \Comment{If only one pixel is a surface}
          \State $pix_{info}.loc(2,:) \gets$ []
          \State $surface(:,2) \gets$ []
          \State $surface_{no} \gets surface_{no} + 1$
          \State $surface_{list}\{surface_{no}\} \gets pix_{info}.loc$
          \Statex Mean per dimension as a feature for just one pixel
          \State $featval \gets$ []
          \For{$i = 1:\mathcal{N}$}
               \State $featval \gets$ [$featval$; $\mu$($surface(i,:)$)]
          \EndFor
          \State $feature_{list}\{surface_{no}\} \gets featval$
          \State $s_{idx} \gets s_{idx} + 1$
     \EndIf
\EndWhile
\EndProcedure
\end{algorithmic}
\end{algorithm}
\begin{table}[!t]
\begin{center}
\begin{tabular}[c]{|c|c|l|}
\hline
\textbf{Rank} & \textbf{$\mathcal{F}$-Score} & \textbf{Algorithm} \\
\hline
$0$ & $0.79$ & Humans \\
$1$ & $0.63$ & Expectation Maximization (EM) \\
$2$ & $0.59$ & Multivariate Chebyshev (mTch) \\
$3$ & $0.58$ & Mean Shift (mShift) \\
$4$ & $0.50$ & Normalized Cuts (nCuts) \\
$5$ & $0.28$ & Graph Based (grBase) \\
\hline
\end{tabular}
\end{center}
\caption{Summary table of rank of algorithms based on
  $\mathcal{F}$-score generated on $100$ test images in the BSB dataset \cite{Martin:2001}.} 
\label{table:f_score}
\end{table}
\begin{figure}[pth]
\begin{center}
\includegraphics[width=8.5cm,height=20cm]{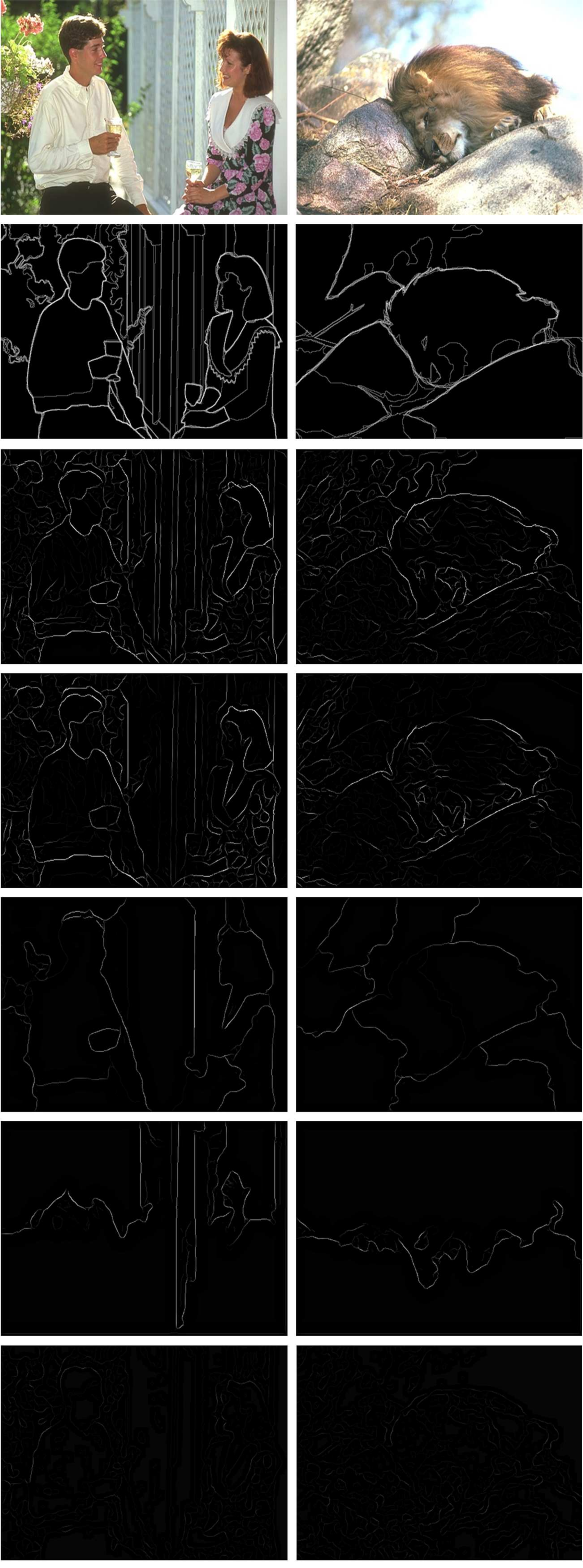}
\caption{Row wise top to bottom: Original image, probabilistic
  boundaries for human segmentation, mTch ($Npar = 0.95$, $\epsilon =
  4$ and number of clusters $10$), EM (number of clusters $10$), nCuts
  (number of clusters $10$), grBase ($\sigma = 0.8$, $k = 300$) and
  mShift ($h_{r} = 8$, $h_{s} = 8$).}
\label{fig:bdbs_01}
\end{center}
\end{figure}
\begin{figure}[pth]
\begin{center}
\includegraphics[width=8.5cm,height=20cm]{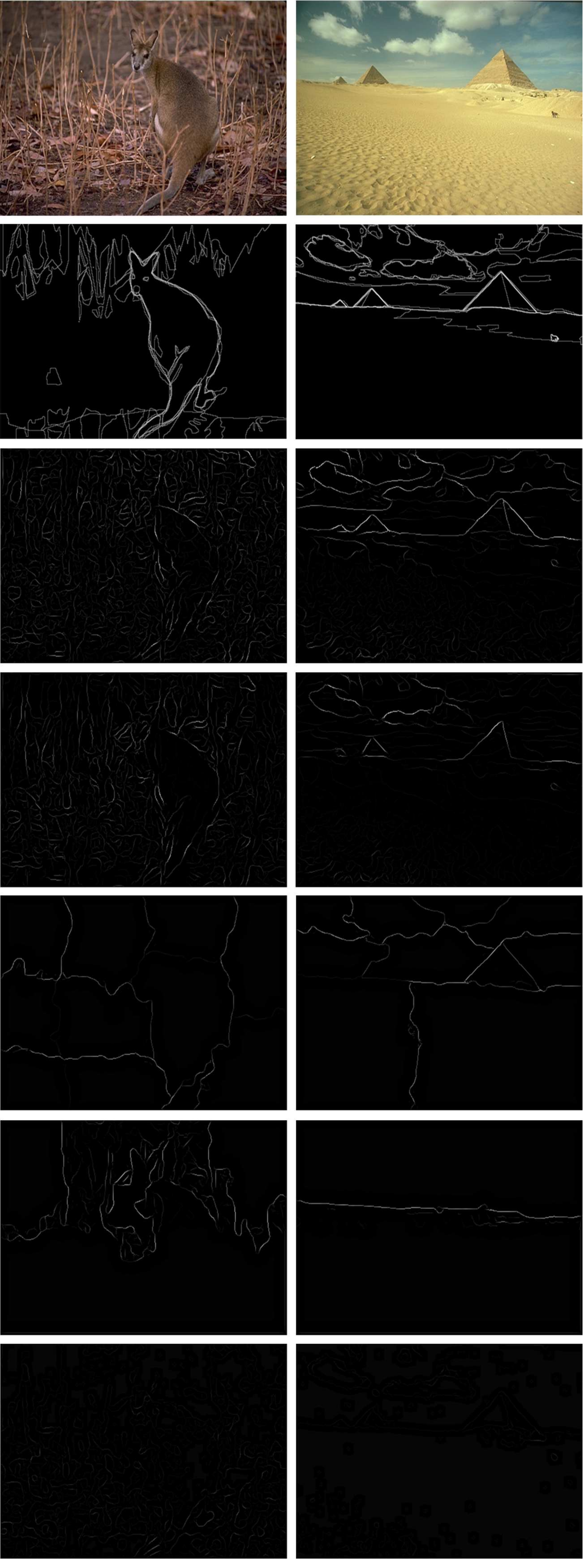}
\caption{Row wise top to bottom: Original image, probabilistic
  boundaries for human segmentation, mTch ($Npar = 0.95$, $\epsilon =
  4$ and number of clusters $10$), EM (number of clusters $10$), nCuts
  (number of clusters $10$), grBase ($\sigma = 0.8$, $k = 300$) and
  mShift ($h_{r} = 8$, $h_{s} = 8$).}
\label{fig:bdbs_02}
\end{center}
\end{figure}

\section{Experiments}\label{sec:experiments}
To test the effectiveness of the multivariate Chebyshev algorithm (mTch), the results obtained on image segmentation were compared with some of the standard existing algorithms. The other algorithms employed were the normalized cuts (nCuts) \cite{Shi:2000}, \cite{Cour:2005}, edge graph based segmentation (grBase) \cite{Felzenszwalb:2004}, the mean shift clustering (mShift) \cite{Comaniciu:2002} and an expectation maximization (EM) based implementation \cite{Herrera:2006}. Figure \ref{fig:pepper_seg} shows the standard peppers image that was segmented into $10$ clusters for mTch, EM and nCuts. For the case of grBase ($\sigma = 0.8$, $k = 300$) and mShift ($h_{r} = 8$, $h_{s} = 8$), parameters values were specified as mentioned in literature. Matlab implementations of grBase and mShift were taken from \cite{Woodford:2009} to produce the results. All algorithms were tested with a fixed parameter value on images with varying content.  This does imply that results per image may not be optimized and would sound a bit unfair, but from another perspective such an experiment also suggests how robust an algorithm is against the variance in content of images, given a fixed parameter value. This outlook holds true when both the grBase and mShift algorithms in this experimental setup donot fair well on fixed parameter value for all images as compared to the proposed mTch algorithm. The $\mathcal{F}$-scores later generated bolster this claim. \par
Probabilistic boundaries were generated based on brightness and texture gradients \cite{Martin:2001} on the segmented images generated from the algorithms under study as well as the available human segmentation. The boundaries were later evaluated to find $\mathcal{F}$-scores. In many of these images, the segmentations based on bounded surfaces gave the good and consistent results. Table \ref{table:f_score} shows the summary of $\mathcal{F}$-scores by the algorithms onthe benchmark. It states that the segementations from the bounded surfaces gave better results than some of the standard well known algorithms. The mTch works second best to EM. A reason for low performace of mTch w.r.t EM may be due the ignorance of important features like surface gradient and reflectance propertices that may add to more discriminative information than using average values of intensities at the current stage. Although the EM gives good performance many times, there are places where it does not capture the nature of surface well. One such case is the texture of the purple cloth and wrinkles in it, in the peppers image (figure \ref{fig:pepper_seg}). For more such cases, the generated results can be made available. \par
\begin{figure}[!t]
\begin{center}
\includegraphics[width=8.5cm,height=10cm]{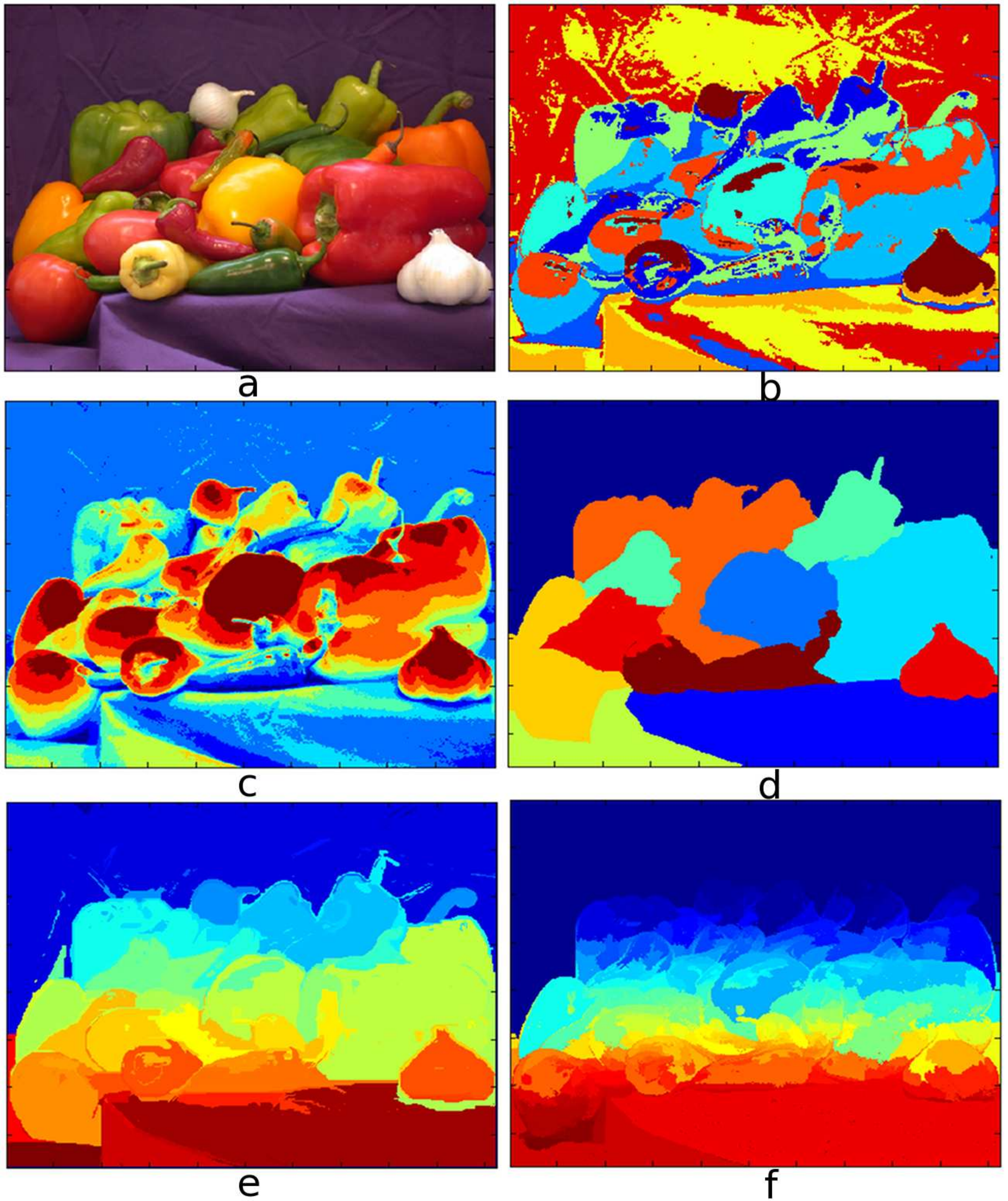}
\caption{(a) Original and segmented images from (b) mTch ($Npar =
  0.95$, $\epsilon = 4$ and number of clusters $10$), (c) EM (number of
  clusters $10$) (d) nCuts (number of clusters $10$) (e) grBase
  ($\sigma = 0.8$, $k = 300$) and (f) mShift ($h_{r} = 8$, $h_{s} = 8$).}
\label{fig:pepper_seg}
\end{center}
\end{figure}
It is widely known that the mShift yields good segments such that no local information is left behind. Too much information sometimes may not be necessary while segmentation are been compared. Figures \ref{fig:bdbs_01} and
\ref{fig:bdbs_02} show the boundaries generated using the routines in the benchmark. Note that in these images, the mTch gave the best results. Also, at first sight it may appear that the boundaries may not have been generated well in case of mShift and grBase, but this is not the case as careful investigation does show their presence. \par
The grBase \cite{Felzenszwalb:2004} works on the basis of a predicate that measures the evidence of a boundary on two ideas, namely: the comparison of intensity differences across boundary and the other, intensity differences among neighbourhood pixels within a region. For this a threshold function is proposed which depends on the components and their respective size. This also means that the parameter used for scaling components would be different for different images, if good segmentation results are desired in grBase. In comparison, the mTch decomposes the images into surfaces which are bounded surfaces or components that preserve the homogeneity in texture using the image content invariant HSFC and the multivariate measure that captures the local neighbourhood interaction in the Tchebyshev inequality. Thus the constant value of parameters for mTch would give good results for different images most of the time. Results in the benchmark dataset prove the issue that for same values of paramters in grBase the segmentation results are inferior to that of mTch of a sample of $100$ test images. \par
It must be noted that the parameter $k$ used in \cite{Felzenszwalb:2004}, scales the size of the component and is not
the minimum component size. Thus it is taken as a constant and no relations of it are derived with respect to the interaction present in the image. Though on similar lines, $\epsilon$ in mTch is also a parameter that defines the degree of control over size of components, but with a bound. The $\epsilon$ characterizes the size of a surface or component
probabilistically, while relating to the control of the degree of texture interaction using the mutivariate measure in the Tchebyshev inequality (equation \ref{equ:ti_imp}). This probabilistic bound is the key relation between the size of component and the texture interaction within, that decomposes the image into homogeneous surfaces in $\mathcal{N}$D by lemma \ref{lem:03}. Clustering on a bunch of homogeneous surfaces is bound to give robust segmentation results. The price that is paid is in terms of time required to cluster the surfaces using the standard $k$-means. \par
Figures \ref{fig:bdbs_01} and \ref{fig:bdbs_02} show the boundaries generated using the routines in the benchmark which are later evaluated to generate $\mathcal{F}$-scores that determine the accuracy of the proposed algorithm as well as that of the other algorithms. Note that in these images, the mTch gave the best results. Also, at first sight it may appear that the boundaries may not have been generated well in case of mShift and grBase, but this is not the case as careful investigation does show their presence. Figure \ref{fig:pr_01}.a and \ref{fig:pr_01}.b shows the precision recall curve for (EM, mTch) and (mTch, mShift), respectively. Figure \ref{fig:pr_02}.a and \ref{fig:pr_02}.b shows the curve for (mTch,
nCuts) and (mTch, grBase), respectively. \par
\begin{figure}[!t]
\begin{center}
\includegraphics[width=8.5cm,height=14cm]{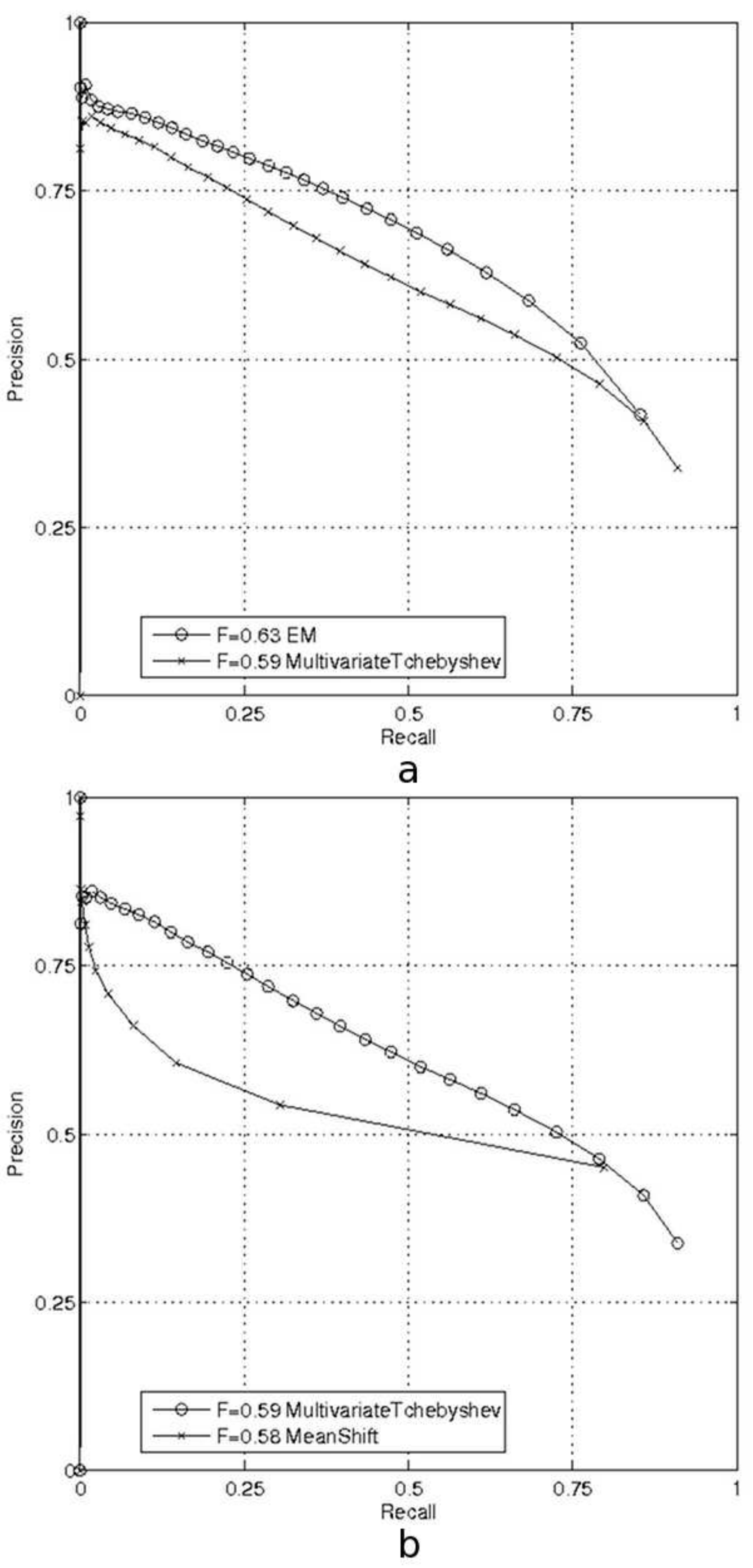}
\caption{Precision Recall curve for (a) EM (number of clusters $10$) and mTch ($Npar =
  0.95$, $\epsilon = 4$ and number of clusters $10$) and (b) mTch
  ($Npar = 0.95$, $\epsilon = 4$ and number of clusters $10$) and
  mShift ($h_{r} = 8$, $h_{s} = 8$).}
\label{fig:pr_01}
\end{center}
\end{figure}
\section{Discussion}\label{sec:discussion}
The proposed method has advantages as well as disadvantages. This section gives an analysis of the intricate points of the algorithm as well as hints as to where improvements can be made. \par 
\emph{Neighbourhood Information} is implicitly considered using a HSFC traversal. The \emph{topography of
  surface} is captured by treating the $\mathcal{N}$ dimensions of the image. For RGB it would be $\mathcal{N} = 3$. \par
\emph{Filtering} of image is not required unless absolutely necessary. The selection of surfaces and checking their validity based on variation of points in $\mathcal{N}$D via Tchebyshev Inequality, makes the operation of filtering computationally redundant. Especially the variation or standard deviation in the $\mathcal{N}$D points is the measure that tolerates the amount of noise that can be taken into account. If the variation is too high, then the initialized surface
may be invalid for further processing due to excessive noise. \par
The \emph{Leakage problem} is that it is not known when to stop to determine the size of the area. This is tackled by determining the length of the surface via the use of the Tchebyshev's Inequality, instead of using thresholds based on image intrinsic intensity values. One aspect that affects the performance of the algorithm is the $k$-means clustering which may require several iterations as well as replicates in order to converge and produce clusters without getting stuck in some local minima. In general terms, if one is not bogged down by the intricacies of the $k$-means then the whole framework works well on surfaces and gives nice segmentation results on the benchmark. The $k$-means on pixels is slower than the $k$-means on the surfaces themselves as the sample size of the former is reduced by a factor of $\epsilon$ while
retaining the local properties using the Hilbert space filling curve. \par
\begin{figure}[!t]
\begin{center}
\includegraphics[width=8.5cm,height=14cm]{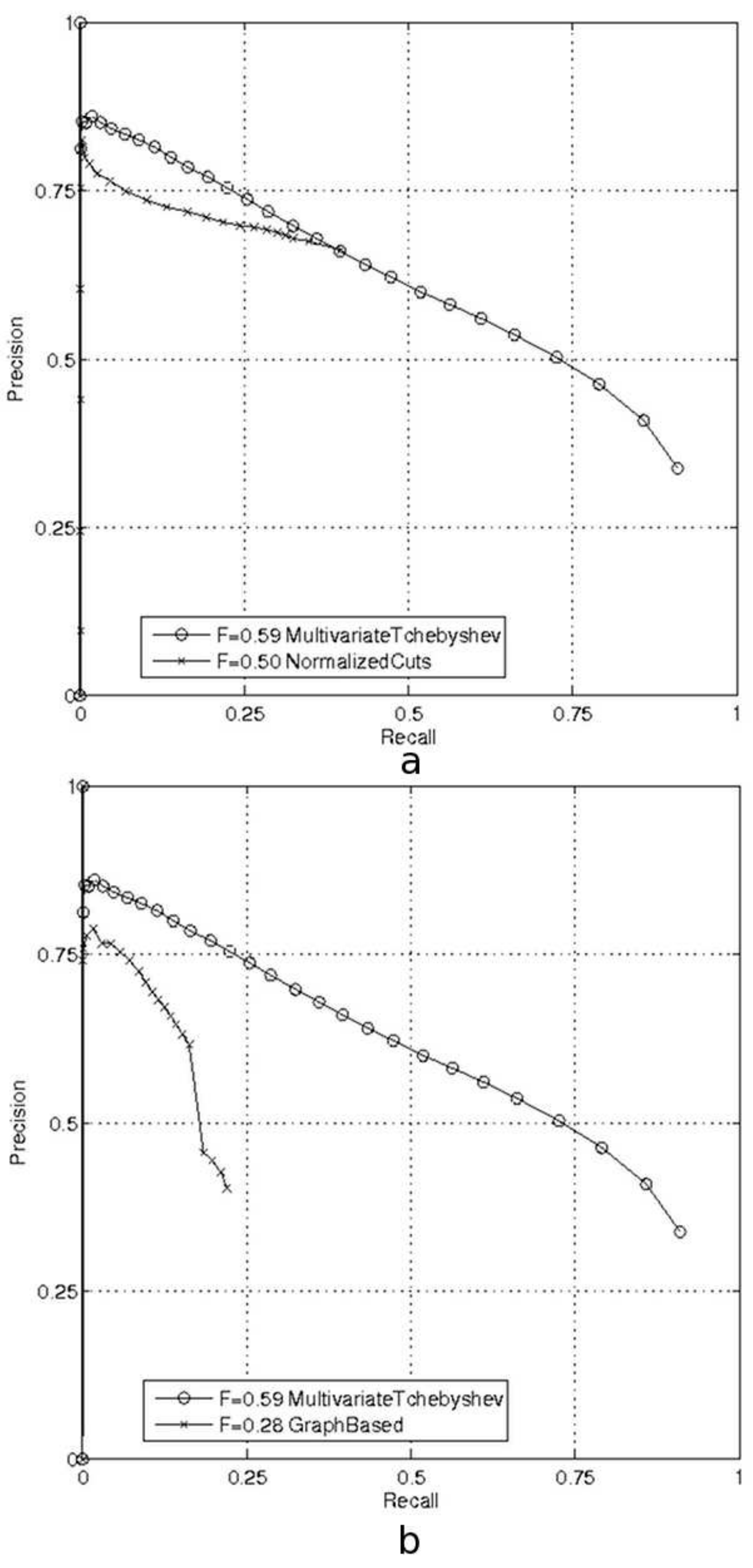}
\caption{Precision Recall curve for (a) mTch ($Npar = 0.95$, $\epsilon
  = 4$ and number of clusters $10$) and nCuts (number of clusters
  $10$) (b) mTch ($Npar = 0.95$, $\epsilon = 4$ and number of clusters
  $10$) and grBase ($\sigma = 0.8$, $k = 300$).}
\label{fig:pr_02}
\end{center}
\end{figure}

The current version of the algorithm also does not optimize the values of $\epsilon$ in equation \ref{equ:ti_imp} and $Npar$ the initialization parameter. Silhouette validation \cite{Rousseeuw:1987} could be one of the methods employed to find the best value of $k$ clusters. The current research does not focus on the number of clusters per se. Instead of focusing as optimisation problem, $Npar$ and $\epsilon$ act as parameters of degree of control over the inclusion of points for surface initialization and size of the surface. The experiments prove that the initial model, without much tunning gives comparable results for segmentation purpose. Robust performance across images with varying context point towards the benefits of using a hybrid model that currently uses a fixed Chebyshev parameter value, a simple measure of similarity and a fixed space filling curve. Intuitively, it can be infered that clustering on these probabilistically bounded homogeneous surfaces will be faster than clustering on pixels. This is because the homogeneous intensity values gets bunddled up together and reduces the sample size of the original problem by a factor of $\epsilon$ (Chebyshev parameter). Because of its generalized framework, the proposed decomposition algorithm can find its application in areas like the generation of textures, combining information from multimodal sources as in biomedical images and processing multidimensional information on space filling curves, to name a few. \par

\section{Conclusion}\label{sec:conclusion}
A novel hybrid model for image decomposition has been proposed. The model works on reduced image representation based on monovariate functions and processes information spread in a multidimensional framework. Initial segmentation results indicate the efficacy of the tool in terms of generalisation and robustness across images with varying content.\par

\section*{Acknowledgement}\label{sec:ack}
The author thanks the Neuroimaging Center (UMCG, Groningen, The Netherlands) for supporting this work and Dr. Charless Fowlkes, assistant professor at Computer Science Department UC Irvine (USA), for sharing the compiled Berkeley Segmentation Benchmark code for Mac intel hardware. \par

%--------------------------------------------------------------------------

\end{document}